%% file: _main.tex
\documentclass[letterpaper, 10 pt, conference]{ieeeconf}  
\IEEEoverridecommandlockouts                              
\usepackage{multicol}
\usepackage[bookmarks=true]{hyperref}

\usepackage{graphicx}

\usepackage{amsthm}
\usepackage{amsmath}
\usepackage{amssymb}
\usepackage{changepage}
\usepackage{comment}
\usepackage{graphbox}

\usepackage{caption}
\usepackage{subcaption}
\usepackage{cite}

\usepackage[normalem]{ulem}
\usepackage{algorithm}
\usepackage{tikz}
\usepackage{algorithmic}
\usepackage[algo2e]{algorithm2e}
\usepackage{overpic}

\newtheorem{theorem}{Theorem}

\newtheorem{problem}{Problem}
\newtheorem{remark}{Remark}
\newtheorem{definition}{Definition}

\usetikzlibrary{positioning}
\title{\LARGE \bf
Optimal Allocation of Many Robot Guards for Sweep-Line Coverage
}
\author{Si Wei Feng \and Teng Guo \and Jingjin Yu
\thanks{$^{1}$S. W. Feng, T. Guo and J. Yu are with the Department of Computer Science, Rutgers, the State University of New Jersey, New Brunswick, NJ 08901, USA. E-mails:         {\tt\small \{siwei.feng, teng.guo, jingjin.yu\}@rutgers.edu}
This work is supported in part by NSF award IIS-1845888 and an Amazon Research Award.
}%
}

\begin{document}

\maketitle
\thispagestyle{empty}
\pagestyle{empty}


\begin{abstract}
We study the problem of allocating many mobile robots for the execution of a pre-defined sweep schedule in a known two-dimensional environment, with applications toward search and rescue, coverage, surveillance, monitoring, pursuit-evasion, and so on.
The mobile robots (or agents) are assumed to have one-dimensional sensing capability 
with probabilistic guarantees that deteriorate as the sensing distance
increases.
In solving such tasks, a time-parameterized distribution of robots 
along the sweep frontier must be computed, with the objective to minimize 
the number of robots used to achieve some desired coverage quality 
guarantee or to maximize the probabilistic guarantee for a 
given number of robots. 
We propose a max-flow based algorithm for solving the allocation task, 
which builds on a decomposition technique of the workspace as a
generalization of the well-known boustrophedon decomposition. 
Our proposed algorithm has a very low polynomial running time and 
completes in under two seconds for polygonal environments with over $10^5$ 
vertices.
Simulation experiments are carried out on three realistic use cases
with randomly generated obstacles of varying shapes, sizes, and spatial distributions, which demonstrate the applicability and scalability our proposed method.

\noindent Introduction video: \href{https://youtu.be/8taX92rzC5k}{\nolinkurl{https://youtu.be/8taX92rzC5k}}.
\end{abstract}

\input{tex/intro.tex}

\input{tex/preli.tex}

\input{tex/algo.tex}

\input{tex/expr.tex}
\input{tex/conclusion.tex}

\bibliographystyle{IEEETran}
\bibliography{bib/references}

\end{document}

%% file: tex/intro.tex
\section{Introduction}

Searching for a static or moving target in a planar environment is a classic 
problem in robotics \cite{guibas1999visibility, suzuki1992searching, lavalle2000algorithm, stiffler2017persistent, kolling2007graph}. 
The setting applies to many real-world applications, including searching for
lost person/object, checking for potential hazards, and generally, search and rescue tasks conducted in a known environment. 
Research tackling this problem mostly focuses on devising a search plan with
different types of objectives, such as minimizing the total length of the
frontier of the search schedule from the start to the end 
\cite{kolling2017coordinated}, minimizing the number of robots used for 
the search plan in a visibility-based robot sensing model
\cite{megiddo1988complexity}, and so on.

In certain cases, the high-level plan for searching a given region may be already pre-determined and fixed. For example, in search and rescue efforts, a frequently carried out plan is to perform a single sweep of an environment with a marching frontier, which is easy to execute when many participating robots/agents are involved.
The high-level search plan may also be determined by existing algorithms that 
compute only the search frontier.
However, even in the case where the search plan consists of pre-determined 
sweep (frontier) lines, it remains non-trivial to find an optimal organization 
of the mobile robots to execute the search plan, for either minimizing the 
number of robots needed for a given sensing probability requirement or utilizing 
a fixed number of robots to maximize the minimum sensing probability of locating something in the environment. 

\begin{figure}[t]
\vspace{1.5mm}
    \centering

    \begin{subfigure}[t]{0.477\textwidth}
         \centering
         \includegraphics[width=\textwidth]{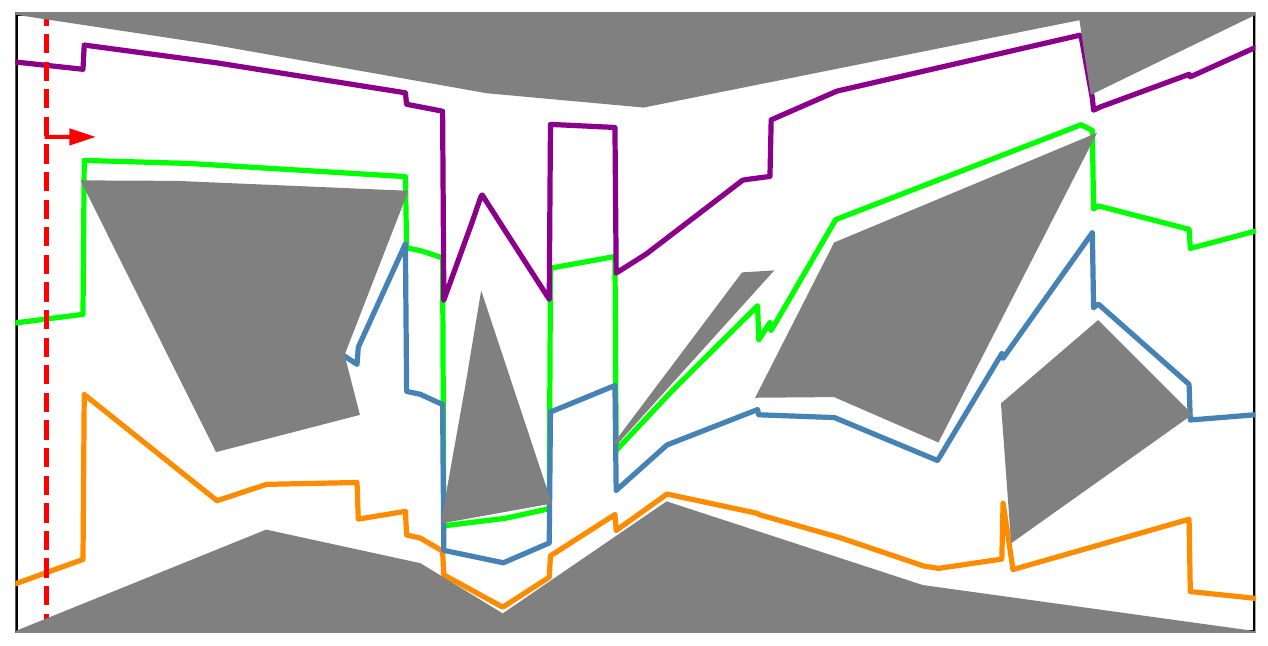}
         \caption{Optimal robot allocations in a vertical sweep}
         \label{fig:vertical}
     \end{subfigure}

    \medskip
    
    \begin{subfigure}[t]{0.23\textwidth}
         \centering
         \includegraphics[width=\textwidth]{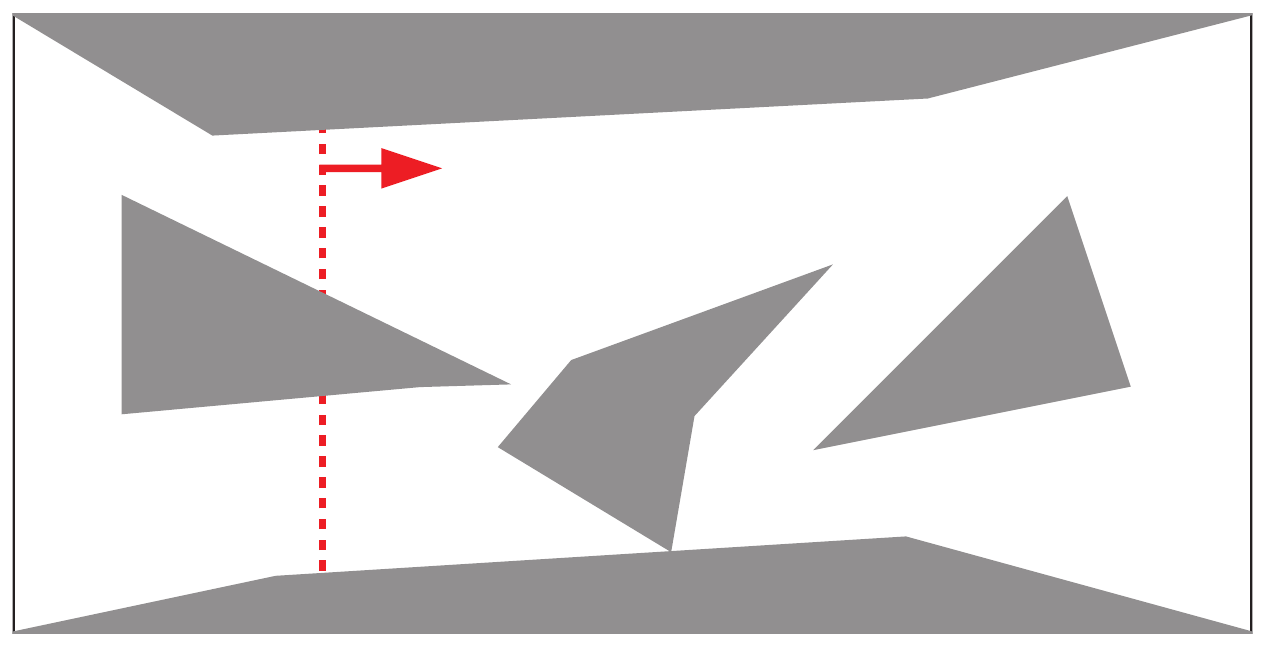}
         \caption{Vertical}
         \label{fig:vertical}
     \end{subfigure}
    \begin{subfigure}[t]{0.115\textwidth}
         \centering
         \includegraphics[width=\textwidth]{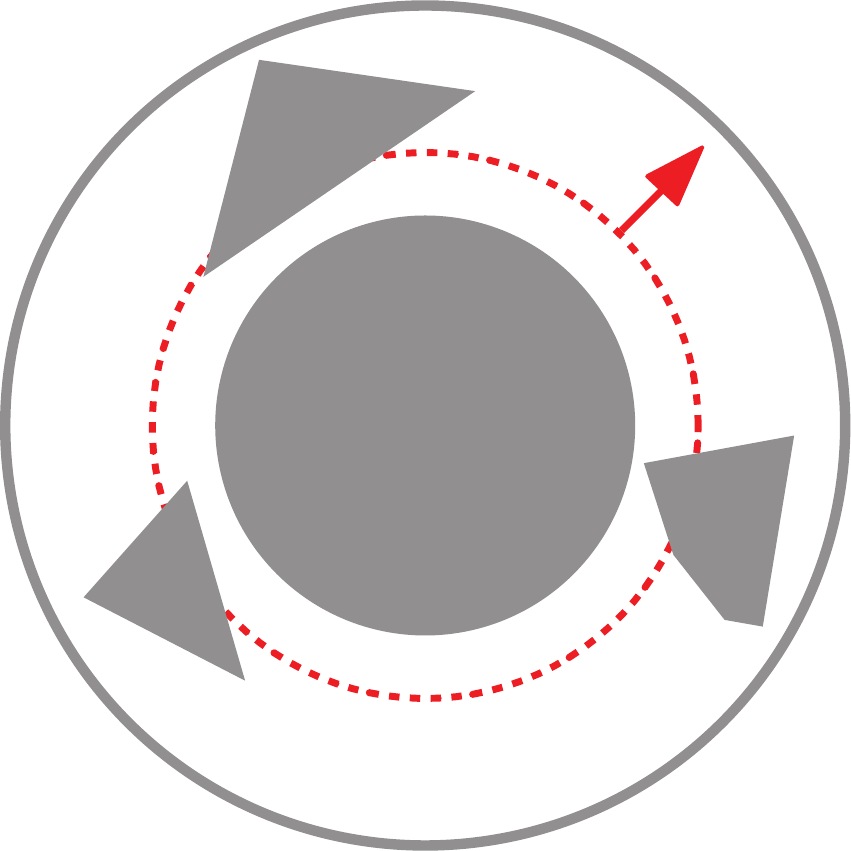}
         \caption{Circular}
         \label{fig:radial}
     \end{subfigure}
    \begin{subfigure}[t]{0.115\textwidth}
         \centering
         \includegraphics[width=\textwidth]{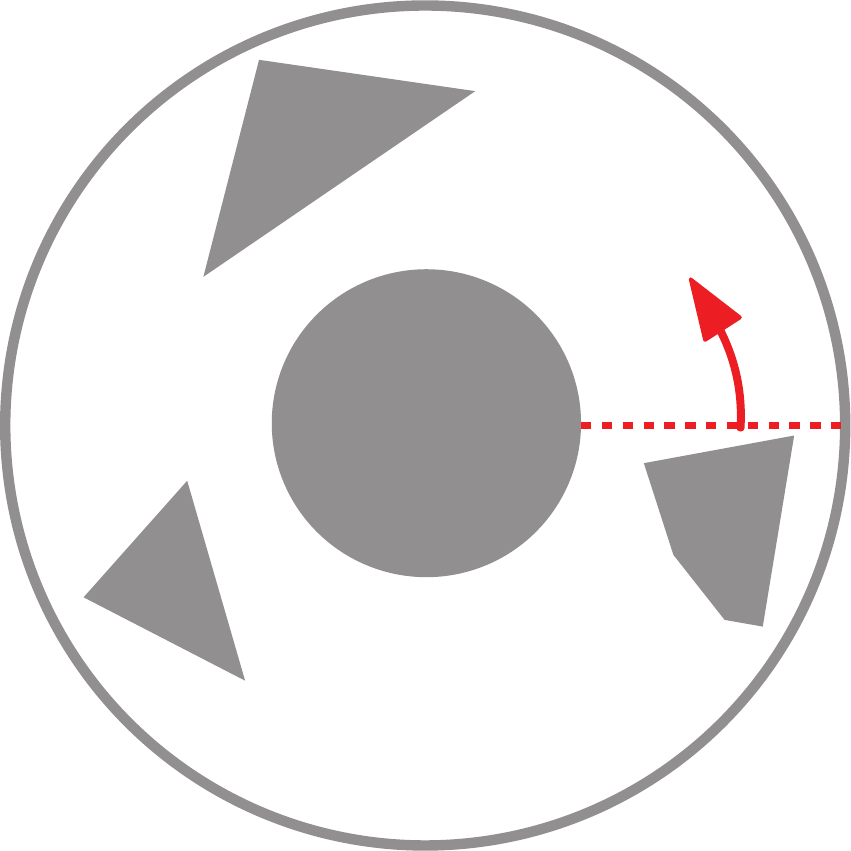}
         \caption{Radial}
         \label{fig:circular}
    \end{subfigure}
     
    \vspace{2mm}
    \caption{(a) An illustration of robots' locations along a left-to-right 
    vertical sweep schedule. Four robots are allocated to execute the vertical sweep, 
    and their trajectories are illustrated in different colors. 
    %
    (b)(c)(d) Illustrations of three use cases: vertical, circular, 
    and radial sweeps.}
    \label{fig:sweep}
\end{figure}

In this work, we address the challenge of how to best allocate 
many robots to execute a pre-determined search schedule for a known environment. 
More specifically, for a two-dimensional closed and bounded workspace, and a known 
search schedule which gives a search frontier for any given time, the robot guards 
are required to stay on the search frontier to carry out the sensing task. 
Because each robot's coverage quality deteriorates with distance, their relative 
placement on the frontier must be carefully decided to maximize the coverage quality. 
For the setup, we focus on the problem of finding the minimum number of robots required 
to execute a pre-determined sweep schedule such that a minimum sensing quality is 
guaranteed for each point in the workspace. Solutions to this minimization problem 
readily translate to solutions for maximizing the coverage quality for a fixed number 
of robots, which is a dual problem. 

In summary, the main contributions of this work are twofold. First, we generalize the 
notion of boustrophedon decomposition \cite{choset2000coverage}, which partitions the 
plane via vertical sweep, into a decomposition of the plane with a \emph{continuous 
monotone sweep schedule}, which we prove can always be represented as a directed 
acyclic graph (DAG). 
Then, we show the problem of minimizing the number of line guards required to execute a sweep schedule can be transformed into a network flow problem.
Since the generalized boustrophedon decomposition and the network flow problem can 
be solved in low polynomial time, our method achieves high levels of scalability.
The strengths of our method are further corroborated in extensive simulation experiments 
on three realistic use cases: \emph{vertical sweep}, \emph{circular sweep}, and 
\emph{radial sweep} (see, e.g., Fig.~\ref{fig:sweep} (b)(c)(d)).

\noindent
\textbf{Related Work.}
The study in this paper draws inspiration from the study of several 
lines of related problems. 
The Graph-Clear problem, formulated in \cite{kolling2007graph}, tasks a group of robots to search and clear an environment with the operations of blocking and clearing.
A follow-up work on Line-Clear \cite{kolling2017coordinated} uses line guards
with more focus on computational geometry in that
the objective is to minimize the maximum sweep line distance. Both of these problems are
NP-hard, establishing the difficulties of finding a sweep schedule for a planar environment.
The more general pursuit-evasion problem dates back to the research on \emph{search number}
on a discrete graph \cite{megiddo1988complexity}, 
followed by studies on pursuit and evasion continuous environment with 
visibility-based model \cite{guibas1999visibility, suzuki1992searching, lavalle2000algorithm, stiffler2017persistent}. 
The problem becomes the well-known art-gallery problem \cite{o1987art} when a static deployment of robots is sought after.
When working with known patrolling search frontiers, e.g., vertical sweep lines, 
this problem is analogous to the perimeter defense problem by placing guards on a static perimeter
to defend intruders \cite{shishika2020cooperative, macharet2020adaptive, chen2021optimal}.
Previously, we have also studied a version of static range guard placement problems for securing perimeters and regions \cite{feng2020optimally}.
In contrast to the pursuit-evasion algorithms that deal with searching dynamic and unpredictable targets that could escape, 
coverage planning/control-related algorithms become more suitable for searching or covering predictable or stationary targets,
e.g., room sweeping, pesticide and fertilizer spraying, persistent monitoring and so on \cite{cortes2004coverage, oksanen2009coverage, haksar2020spatial, wei2018coverage, deng2019constrained, lan2013planning, cassandras2012optimal, yu2015persistent, palacios2017optimal}. 

%% file: tex/preli.tex
\section{Preliminaries}

In this section, we first describe a general probabilistic robot sensing model 
used in this paper and introduce the notion of a \emph{continuous monotone 
sweep schedule}, along which robots must be deployed to carry out the 
search task. 
Then, the problem of finding the minimum number of robots for a certain 
sweep schedule is formally defined.

\subsection{Robot Sensing Model}
\label{sec:sensing}

In this work, a robot is assumed to be a \emph{line guard} capable of sensing 
relevant events happening on a continuous line segment passing through the
robot. 
On the line segment guarded by a robot, for a (point) target at a distance 
of $r$ to the robot, the robot has probability $\rho(r)$ of 
detecting or capturing the target, where $\rho: \mathbb{R}^+\rightarrow[0,1]$ 
is a decreasing sensing probability function. 
%

Individual robots' 1D sensing range aligns to form a \emph{sweep line} or \emph{sweep frontier}.
It is assumed that robots carry out their sensing functions independently 
without interfering with each other.
For a point $p$ in the sweep line, if it falls between two robots $r_1, 
r_2$, the coverage of that point is provided by $r_1$ and $r_2$, and the 
probability of target detection at that point is computed as $1-(1-\rho(\ell_1))\cdot(1-\rho(\ell_2))=\rho(\ell_1) + \rho(\ell_2) - \rho(\ell_1) \cdot \rho(\ell_2)$, where $\ell_1$ (resp., $\ell_2$) is the distance between $p$ and $r_1$ (resp., $r_2$).
If it lies at the ends of the segment, the coverage of that point is only provided by the closest robot, and the coverage probability is simply $\rho(\ell)$, where $\ell$ is the distance between $p$ and the robot. These are illustrated in Fig.~\ref{fig:sensing}.
We note that \emph{sweep lines} are not necessarily straight. For example, a sweep line could be part of a circle in the case of circular sweeps (Fig.~\ref{fig:sweep}(c)).

\begin{figure}[ht]
    \centering
    \begin{tikzpicture}

        \draw[black,thick] (0, 0) -- (6, 0);
        
        \filldraw[black] (0.2, 0) circle (1.5pt)
        node[anchor=north]{$p_2$};
        
        \draw[thick, <-] (0.2, 0.2) -- (0.55, 0.2); 
        \node[] at (0.7,0.2) {$\ell_3$};
        \draw[thick, ->] (0.85, 0.2) -- (1.2, 0.2);
        
        \filldraw[black] (1.2, 0) circle (2pt) node[anchor=north]{$r_1$};
        
        \draw[thick, <-] (1.2, 0.2) -- (1.5, 0.2);
        \node[] at (1.65,0.2) {$\ell_1$};
        \draw[thick, ->] (1.8, 0.2) -- (2.1, 0.2);
        
        \filldraw[black] (2.1, 0) circle (1.5pt)
        node[anchor=north]{$p_1$};
        
        \draw[thick, <-] (2.1, 0.2) -- (2.4, 0.2);
        \node[] at (2.55,0.2) {$\ell_2$};
        \draw[thick, ->] (2.7, 0.2) -- (3, 0.2);
        
        \filldraw[black] (3, 0) circle (2pt) node[anchor=north]{$r_2$};
        
        \filldraw[black] (4.8, 0) circle (2pt) node[anchor=north]{$r_3$};

    \end{tikzpicture}
    \caption{Illustration of the robot sensing model. The segment is being covered by three robot guards $r_1 \sim r_3$. The coverage probability of $p_1$, falling between robots $r_1$ and $r_2$, is given as $\rho(\ell_1)+\rho(\ell_2)-\rho(\ell_1)\cdot\rho(\ell_2)$, and the coverage probability of $p_2$, covered only by $r_1$ is given as $\rho(\ell_3)$.}
    \label{fig:sensing}
\end{figure}
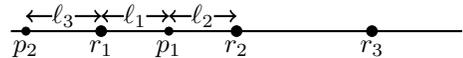

For covering a line segment with a length of $\ell$ and sensing function of $\rho$, 
we denote as $\zeta: \mathbb{R}^+ \rightarrow \mathbb{N}^+$ a primitive for computing 
the minimum number of robots needed to achieve the required minimum covering 
probability of $\rho_0$. 

Take the exponential decaying sensing probability function as an example, 
where $\rho(r) = e^{-c\cdot r}$, $c>0$ is some constant.
In this case, the maximum distance at the two ends to guarantee the sensing probability of $\rho_0$ is $d_1=-\ln (\rho_0)/c$. 
If a point is between two robots with distance $\ell_1$ and $\ell_2$ to the two robots, its sensing probability is bounded by

\begin{align*}
\rho(\ell_1) + \rho(\ell_2) - \rho(\ell_1)\cdot \rho(\ell_2)\geq 2e^{-\frac{c\cdot d}{2}} - e^{-c\cdot d}    \\
where\ d = \ell_1+\ell_2
\end{align*}

So, the required maximum distance between two neighboring robots is $d_2=-2\ln(1-\sqrt{1-\rho_0})/c$.
Therefore, the minimum number of robots required to cover a line segment with length $\ell$ is 
$\zeta(\ell) = \max(1, 1 + \lceil(\ell-2\cdot d_1)/d_2 \rceil )$.

\subsection{Problem Formulation}
In this paper, we work with a 2D compact (i.e., closed and bounded) workspace 
${\mathcal W} \subset \mathbb{R}^2$, which can contain a set of obstacles. 
An existing \emph{sweep schedule} is an input to the problem, which we
define the \emph{continuous monotone sweep schedule} for $\mathcal W$ as 
\begin{definition}[Continuous Monotone Sweep Schedule]
A function $P(t)$ that maps a positive timestamp $t$ to a continuous curve is a continuous monotone sweep schedule for $\mathcal W$ if $P(t)$ changes continuously, 
and for each point $p\in \mathcal W$, there exists a single $t'$ such that $p\in P(t')$.
\end{definition}

Intuitively, a continuous monotone sweep schedule defines a function 
that maps the time step to a 1-D curve in the workspace that sweeps 
through each point in $\mathcal W$ only once. 
Depending on the continuous curve, $P(t)$ could take various forms. For example, the sweep line may be straight in a search-and-rescue scenario. Or the sweep line may be circular in a search-and-capture scenario. 
In the rest of this paper, we simply refer to a continuous monotone sweep schedule as a \emph{sweep schedule}. Next, we introduce the notions of \emph{arrival time} and \emph{monotone chain}.
\begin{definition}[Arrival Time]
For a search schedule $P(t)$, and a point in the workspace $o\in\mathcal W$,
the arrival time at $o$, $arrival(o)$, is defined as the time step $t$ when $o\in P(t)$.
\end{definition}

\begin{definition}[Monotone Chain]
For a bounded 2D chain: $\tau(s): [0,1]\rightarrow\mathcal{W}$, where $\tau$ is a continuous function, it is considered as a monotone chain to the search schedule $P(t)$ if and only if
\[s_1 < s_2 \Leftrightarrow arrival(\tau(s_1)) < arrival(\tau(s_2)).\]
\end{definition}
In a search schedule or plan, $P(t)$ may be intersected by obstacles. In this case, $P(t)$ are separated into multiple continuous segments; each of these segments requires a dedicated group of robots. That is, robots on one continuous segment cannot provide coverage of other segments. 


With the probabilistic robot sensing model, we formulate the robot team scheduling problem for a given sweep schedule as the following,
\begin{problem}
Given a sweep schedule $P(t)$ for a 2D region $\mathcal W$, and a group of robots with coverage capability function $\rho$, what is the minimum number of robots required to execute the sweep schedule on the sweep line, such that 
for every point $o$ in $\mathcal W$, the probability that $o$ is covered is at 
least some fixed $0 < \rho_0 \le 1$?
\end{problem}

%% file: tex/algo.tex
\section{Optimal Robot Allocation}

Our proposed algorithm can be divided into two steps at the high level. 
In the first step, the algorithm conducts a \emph{generalized boustrophedon 
decomposition} for the given environment along the sweep line, 
which generates a \textit{directed acyclic graph} (DAG) representation 
of the workspace $\mathcal W$ for a given sweep schedule. 
Using the DAG, a max-flow based algorithm is then applied to compute the 
minimum number of robots required for executing the sweep schedule, 
as well as the corresponding arrangement of robots.

\subsection{Generalized Boustrophedon Decomposition}
In solving search and coverage problems, various decomposition techniques 
have been proposed, including trapezoidal, Voronoi, boustrophedon, Morse decompositions, 
and so on \cite{huang2001optimal, choset2000coverage, breitenmoser2010voronoi, acar2002morse}.
For our robot allocation task, it is also natural to start with a decomposition 
of the environment. 
However, we need a decomposition supporting non-straight 
boundaries between the decomposed cells, created by the sweep schedule. 
For this, we propose a generalization of boustrophedon decomposition. 

Before describing the generalization of boustrophedon decomposition, 
we briefly introduce boustrophedon decomposition (readers are referred to 
\cite{choset2000coverage} for further details), which in turn is based on 
trapezoidal decomposition. The difference is that it removes the sweeping 
events that cross inner vertices. As illustrated in Fig.~\ref{fig:trebou},
compared with trapezoidal decomposition, boustrophedon decomposition 
has fewer cells, which leads to less (back-and-forth) boustrophedon motions.

\begin{figure}[ht]
    \centering
    \begin{tikzpicture}
    \fill[gray] (0, 0) -- (0.4, -0.5) -- (0.6, -0.6) -- (0.8, -0.5) -- (1.1, 0.1) -- (0.6, 0.5) -- (0.3, 0.5) -- cycle;
    \draw (-1, - 1) -- (2, -1); 
    \draw (-1, 1) -- (2, 1); 
    \draw[dashed] (0, -1) -- (0,1);
    \draw[dashed] (0.4, -1) -- (0.4,-0.5);
    \draw[dashed] (0.6, -1) -- (0.6,-.6);
    \draw[dashed] (0.8, -1) -- (0.8,-0.5);
    \draw[dashed] (1.1, -1) -- (1.1,1);
    \draw[dashed] (0.6, 0.5) -- (0.6,1);
    \draw[dashed] (0.3, 0.5) -- (0.3,1);
    
    \node[text=black] at (-0.5, 0.0) {cell};
    \node[text=black] at (1.5, 0.0) {cell};
    
    \fill[gray] (4, 0) -- (4.4, -0.5) -- (4.6, -0.6) -- (4.8, -0.5) -- (5.1, 0.1) -- (4.6, 0.5) -- (4.3, 0.5) -- cycle;
    
    \draw (3, - 1) -- (6, -1); 
    \draw (3, 1) -- (6, 1); 
    
    \draw[dashed] (4, -1) -- (4, 1);
    \draw[dashed] (5.1, -1) -- (5.1, 1);

    \node[text=black] at (3.5, 0.0) {cell};
    \node[text=black] at (4.5, -.8) {cell};
    \node[text=black] at (4.5, 0.7) {cell};
    \node[text=black] at (5.5, 0.0) {cell};

    \end{tikzpicture}
    \caption{[left] A trapezoidal decomposition, creating a total of $9$ cells. [right]
    Boustrophedon decomposition of the same environment, creating only $4$ cells.}
    \label{fig:trebou}
\end{figure}
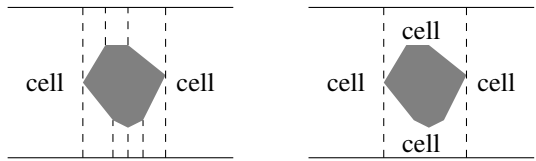

We extend the boustrophedon decomposition from using only vertical sweep lines 
to allowing the use of any \emph{continuous monotone sweep schedule}.
We are given a sweep schedule $P(t)$ for a workspace $\mathcal W$, 
which could take any curved form (see Fig.~\ref{fig:Bou}).
By following the sweep schedule, it is possible to
construct a decomposition of $\mathcal W$, based on the events of cell
splitting and merging. 

\begin{figure}[ht]
    \centering
    \begin{overpic}[width = .4\textwidth]{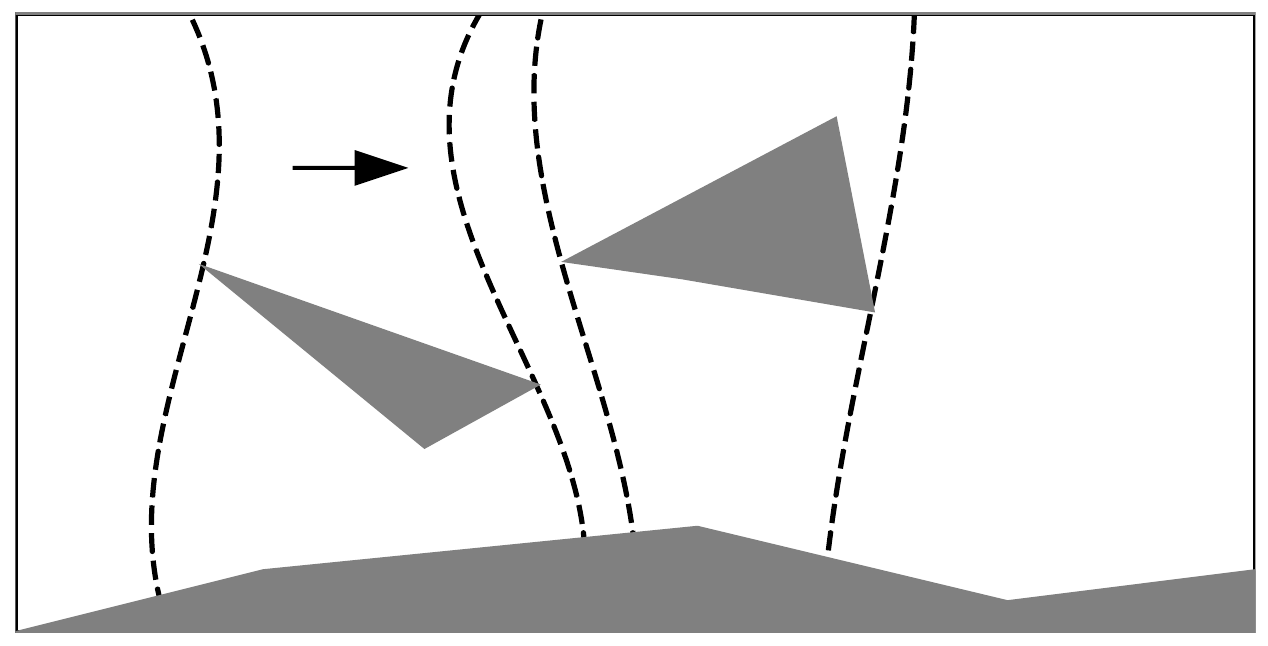}
    \put(8, 35){$v_1$}
    \put(20, 10){$v_2$}
    \put(20, 30){$v_3$}
    \put(36.5, 42){$v_4$}
    \put(55, 20){$v_5$}
    \put(50, 40){$v_6$}
    \put(80, 30){$v_7$}
    \end{overpic}
    \caption{Suppose we have a sweep schedule that sweeps the environment from left 
    to right, with curved sweep fronts. 
    The curves, as they cross critical vertices of objects, are shown as the 
    dashed lines.
    The generalized boustrophedon decomposition decomposes and $\mathcal W$ into 
    $7$ cells, $v_1\sim v_7$. 
    It is important to note here that, for any continuous monotone sweep schedule,
    a decomposition can be obtained. 
    }
    \label{fig:Bou}
\end{figure}
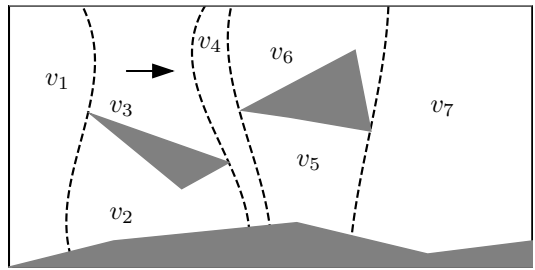

\begin{theorem}
A sweep schedule can be organized in a DAG by the generalized boustrophedon decomposition.
\end{theorem}

\begin{proof}
Following the sweep schedule, we can conduct a generalized boustrophedon 
decomposition of the workspace $\mathcal W$. 
A node in the DAG will represent a cell after the decomposition, whose 
parents and children are the predecessor and successor cells along the 
sweep schedule. 

Additionally, we add a source node which links to all nodes without a parent 
and add a terminal node which links from all nodes without a child.
Since obstacles in the environment can contain concave vertices, we must 
consider two special cases involving concave vertices for the construction of DAG, 
as illustrated in Fig.~\ref{fig:concave_vertices}. 
When a line segment of the sweep schedule ends at a concave vertex, 
the corresponding node in the DAG is linked to the terminal node $t$.
When a line segment of the sweep schedule starts at a concave vertex,
the corresponding node in the DAG is linked from the source node $s$.
In mapping these scenarios to detailed plans, it means that 
some robots will start later or end earlier compared with 
the others.

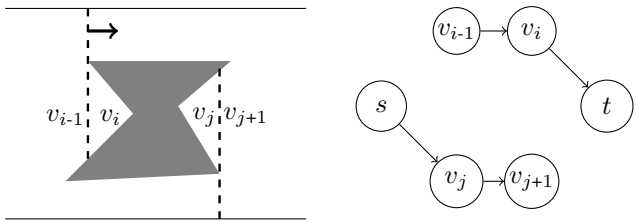
\begin{figure} [ht]
    \centering
    \begin{tikzpicture}
    \fill[gray] (-0.2, 0) -- (1.85, 0.1) -- (1.3, 1) -- (2.0, 1.6) -- (0.1, 1.6) -- (0.7, 0.9) -- cycle;
    \draw (-1, - 0.5) -- (3, -.5); 
    \draw (-1, 2.3) -- (3, 2.3); 

    \draw[dashed, line width=0.9pt, black] (0.1, 2.3) -- (0.1, 0.3);
    \draw[dashed, line width=0.9pt, black] (1.85, 1.5) -- (1.85, -0.5);

    \draw(-0.2, 0.85) node[text=black] { $v_{i\text{-}1}$};
    \draw(0.4, 0.85) node[text=black] {$v_i$};
    \node[text=black] at (1.65, 0.85) {$v_j$};
    \node[text=black] at (2.2, 0.85) {$v_{j\text{+}1}$};
    
    \draw(4, 1.0) node [circle, radius = 0.12, draw, inner sep=4.5](s){ $s$};
    
    \draw(5, 2.0) node [circle, radius = 0.12, draw, inner sep=0.8](vi1){ $v_{i\text{-}1}$};
    \draw(6, 2.0) node [circle, radius = 0.12, draw, inner sep=3](vi){ $v_{i}$};
    
    \draw(5, 0.0) node [circle, radius = 0.12, draw, inner sep=3](vj){ $v_{j}$};
    \draw(6, 0.0) node [circle, radius = 0.12, draw, inner sep=0.8](vj1){ $v_{j\text{+}1}$};
    
    \draw(7, 1.0) node [circle, radius = 0.12, draw, inner sep=4.5](t){ $t$};
    
    \draw[->] (vi) -- (t);
    \draw[->] (vi1) -- (vi);
    
    \draw[->] (vj) -- (vj1);
    \draw[->] (s) -- (vj);
    
    \draw[very thick, ->] (0.1, 2.0) -- (0.5,2.0);
    \end{tikzpicture}
    \caption{An example that contains two concave scenarios. As a result of the DAG construction, some robots will
    start their work at $v_j$ (by having $s$ as its parent) and some robots will end their work at $v_i$ (by having $t$ as its child).}
    \label{fig:concave_vertices}
\end{figure}
\end{proof}


The implementation of the generalized boustrophedon decomposition is similar 
to vertical decomposition, we provide here an implementation 
Alg.~\ref{alg:genbou} adapted from \cite{lavalle2006planning} under the 
\emph{generation position} assumption (i.e., there are no degenerative 
settings).
Basically, the algorithm works by maintaining a binary search tree for 
\emph{boundary chains} representing cell boundaries. The algorithm considers 
two types of events during the sweeping process: splitting a cell into two 
cells and merging two cells into one cell.
Since the time cost mainly comes from maintaining the binary search tree of 
chains, with an efficient implementation, Alg.~\ref{alg:genbou} requires 
$O(n \log n)$ time, where $n$ is the complexity of the environment.

\begin{algorithm}[ht]
\SetKwFunction{genbou}{GenBouDecomp}
\KwData{$P(t)$: a sweep schedule. $\mathcal W$: the workspace.}
\KwResult{a DAG from the decomposition.}
\vspace{1mm}
\SetKwComment{comment}{\%}{}
\DontPrintSemicolon
\SetKw{arrival}{Arrival}

1. Separate the boundaries of obstacles into a set of \emph{monotone chains},
where each chain has the points on it \emph{arrival time} arranged in an 
increasing order.;\;
\vspace{0.5mm}

2. $T\gets$ A binary search  tree of the chains\;
\vspace{0.5mm}

3. Construct an \emph{event array} that contains the events of the start of a chain and the end of a chain, sorted by their occurrence time during the sweep.\;
\vspace{0.5mm}

4. Iterate over the event array, which inserts and deletes chains from the binary search tree $T$, while making sure that $T$ represents the current order of the chains.\;
\vspace{0.5mm}
\begin{adjustwidth}{0.5cm}{}
\begin{adjustwidth}{0.5cm}{}
{\hspace{-.5cm}4.1.}
When two chains are added to $T$, a cell is split into two new cells in the case of a convex
vertex, or a new cell is created for a concave vertex, where an edge directed from $s$ is added.\;
\end{adjustwidth}
\begin{adjustwidth}{0.5cm}{}
{\hspace{-.5cm}4.2.}
When two chains are removed from $T$, two cells are merged into a new cell in the case of 
a convex vertex, or a cell disappears for a concave vertex, where an edge is added directed to $t$.\;
\end{adjustwidth}
\end{adjustwidth}
\vspace{1mm}
5. Return the DAG constructed based on $P(t)$.
\vspace{1mm}
\caption{\protect\genbou{$P$, $\mathcal{W}$}: Generalized Boustrophedon Decomposition} 
\label{alg:genbou}
\end{algorithm}

\subsection{Reduction to Circulation with Demand}
Since we are working with \emph{monotone sweep schedules}, for any point $p$ in 
$\mathcal W$, it can only be contained in $P(t')$ for a single $t'$, i.e., each 
point $p \in \mathcal W$ is only swept once. 
Denote the length of the line segment where it is contained as $L$, which 
requires at least $\zeta(L)$ robots inside that segment at time $t'$
(recall that $\zeta$ is the primitive defined in Sec.~\ref{sec:sensing}).
For each decomposed cell, the minimum number of robots required is $\zeta(l_{max})$, where $\ell_{max}$ 
is the maximum length of the sweep line inside that cell.
Given a DAG $G(V,E)$ that represents the sweep schedule, the arrangement 
problem of the robots along the sweep line can be transformed into a network 
flow problem on the DAG.
For each node $v\in G$, there is a requirement of coverage for the node, 
$demand(v)$, which can be computed as $\zeta(v.\ell_{max})$.
Given the DAG and the demands, we are then to ``flow'' the robots through 
the schedule, allocating a certain number of robots to each decomposed cell 
along the way to satisfy these demands. 

\begin{figure}[ht]
    \centering
    \begin{tikzpicture}
    \draw(0,0) node [circle, radius = 0.12, draw](v1){ $v_1$};
    \draw(1, -0.5) node [circle, radius = 0.12, draw](v2){ $v_2$};
    \draw(1, 0.5) node [circle, radius = 0.12, draw](v3){ $v_3$};
    \draw(2, 0.0) node [circle, radius = 0.12, draw](v4){ $v_4$};
    \draw(3, -0.5) node [circle, radius = 0.12, draw](v5){ $v_5$};
    \draw(3, 0.5) node [circle, radius = 0.12, draw](v6){ $v_6$};
    \draw(4, 0.0) node [circle, radius = 0.12, draw](v7){ $v_7$};
    
    \draw(-1.2, -.5) node [circle, radius = 0.2, draw,inner sep=4.5](s){ $s$};
    
    \draw(5.2, -.5) node [circle, radius = 0.2, draw,inner sep=4.5](t){ $t$};
    
    \draw[->] (v1) -- (v2);
    \draw[->] (v1) -- (v3);
    \draw[->] (v2) -- (v4);
    \draw[->] (v3) -- (v4);
    \draw[->] (v4) -- (v5);
    \draw[->] (v4) -- (v6);
    \draw[->] (v5) -- (v7);
    \draw[->] (v6) -- (v7);
    
    \draw[->] (s) -- (v1);
    \draw[->] (v7) -- (t);
    \end{tikzpicture}
    \caption{The constructed DAG from the example in Fig.~\ref{fig:Bou}}
    \label{fig:DAG}
\end{figure}
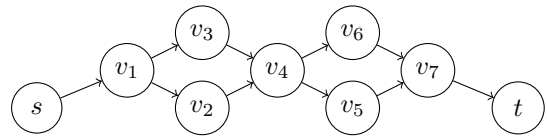

Specifically, our problem may be further cast as a \emph{circulation with demand} 
problem \cite{kleinberg2006algorithm}, with the following augmentation. We replace 
each node $v$ in $G$ with two vertices $v_1$ and $v_2$ and replace every edge 
$uv$ in the previous graph with edge $u_2 v_1$, as illustrated in Fig.~\ref{fig:flow}. 
The edge between $v_1$ and $v_2$ has flow demand of $demand(v)$. 
The following minimum circulation with demand problem is then obtained.

\begin{algorithm}[ht]
\DontPrintSemicolon
\SetKwFunction{mindag}{MinSweepDAG}
\SetKwFunction{dflow}{MinCirculationWithDemand}
\SetKwFunction{addedge}{add\_edge}
\SetKwFunction{addvertex}{add\_vertex}
\KwData{$dag(V, E)$: the DAG obtained from \genbou. $\zeta$: the sensing requirement primitive function.}
\KwResult{$dag$: the updated $dag$ with the robot allocation information}
\vspace{0.5mm}
$G'\gets$ a new empty graph;\;
\vspace{0.5mm}

\For{$v\in dag$}{
\vspace{0.5mm}
    $G'$.\addvertex$(v_1),$ $ G'$.\addvertex($v_2$);\;
\vspace{0.5mm}
    
    $G'.$\addedge($v_1, v2$, capa=$\infty$, demand=$\zeta(v.\ell_{max})$);\;
\vspace{0.5mm}
    
    \For{$u\in dag.neighbor[v]$}{
\vspace{0.5mm}
        $G'$.\addedge($v_2$, $u$, capa=$\infty$, demand=0);\;
    }
}
\vspace{0.5mm}

\dflow($G'$);\;
\vspace{0.5mm}

\For{$v\in dag$}{
\vspace{0.5mm}
    $dag.v.guards\_num\gets$ $G'.flow[v_1][v_2]$;\;
    
\vspace{0.5mm}
    \For{$u\in dag.neighbor[v]$}{
        $dag.flow[v][u] = G'.flow[v_2][u_1]$;\;
    }
    
}

\Return{$dag$}\;


\caption{\protect\mindag{$dag$, $\zeta$}{}}
\label{alg:mindag}
\end{algorithm}

\begin{figure*}[ht]
    \centering
    \begin{tikzpicture}[scale = 1, every node/.style={inner sep=4}]
    \draw(0,0) node [circle, draw, scale = 0.8](v11){ $v_{11}$};
    \draw(1,0) node [circle, draw, scale = 0.8](v12){ $v_{12}$};
    
    \draw(2, -0.5) node [circle, draw, scale = 0.8](v21){ $v_{21}$};
    \draw(3, -0.5) node [circle, draw, scale = 0.8](v22){ $v_{22}$};
    
    \draw(2, 0.5) node [circle, draw, scale = 0.8](v31){ $v_{31}$};
    \draw(3, 0.5) node [circle, draw, scale = 0.8](v32){ $v_{32}$};
    
    \draw(4, 0.0) node [circle, draw, scale = 0.8](v41){ $v_{41}$};
    \draw(5, 0.0) node [circle, draw, scale = 0.8](v42){ $v_{42}$};
    
    \draw(6, -0.5) node [circle, draw, scale = 0.8](v51){ $v_{51}$};
    \draw(7, -0.5) node [circle, draw, scale = 0.8](v52){ $v_{52}$};
    
    \draw(6, 0.5) node [circle, draw, scale = 0.8](v61){ $v_{61}$};
    \draw(7, 0.5) node [circle, draw, scale = 0.8](v62){ $v_{62}$};
    
    \draw(8, 0.0) node [circle, draw, scale = 0.8](v71){ $v_{71}$};
    \draw(9, 0.0) node [circle, draw, scale = 0.8](v72){ $v_{72}$};
    
    \draw(-1, -.5) node [circle, draw](s){ $s$};
    \draw(10, -.5) node [circle, draw](t){ $t$};
    
    \draw[->] (v12) -- (v21);
    \draw[->] (v12) -- (v31);
    \draw[->] (v22) -- (v41);
    \draw[->] (v32) -- (v41);
    \draw[->] (v42) -- (v51);
    \draw[->] (v42) -- (v61);
    \draw[->] (v52) -- (v71);
    \draw[->] (v62) -- (v71);
    
    \draw[->] (v11) -- node[above]{\footnotesize $3$} (v12) ;
    \draw[->] (v21) -- node[above]{\footnotesize $2$} (v22) ;
    \draw[->] (v31) -- node[above]{\footnotesize $2$} (v32) ;
    \draw[->] (v41) -- node[above]{\footnotesize $3$} (v42) ;
    \draw[->] (v51) -- node[above]{\footnotesize $2$} (v52) ;
    \draw[->] (v61) -- node[above]{\footnotesize $2$} (v62) ;
    \draw[->] (v71) -- node[above]{\footnotesize $3$} (v72) ;
    
    \draw[->] (s) -- (v11) ;
    \draw[->] (v72) -- (t) ;
    
    \draw[dashed] (t) .. controls(9, 1.5) and (5, 1.5) .. (4.5, 1.5);
    \draw[dashed,->] (4.5,1.5) .. controls(4,1.5) and (0,1.5) .. (s);
    
    \draw(3, -2) node [circle, draw, dashed] (s'){$s'$};
    \draw(6, -2) node [circle, draw, dashed] (t'){$t'$};
    
    \draw[dashed,->] (s') -- (v41);
    \draw[dashed,->] (v41) -- (t');
    
    \draw[dashed,->] (s') -- (v42);
    \draw[dashed,->] (v42) --  (t');
    
    \end{tikzpicture}
    \caption{Transform the DAG in Fig.~\ref{fig:DAG} into a circulation with demand problem. 
    The value above each edge represents the demand of the edge, which is eliminated when it is zero.
    The source node is $s$, and the target node is $t$. $s'$ and $t'$ are the auxiliary source and target for solving the ``circulation with demand'' problem.
    Also, auxiliary edges are added between $s', t'$ and every other vertex.}
    \label{fig:flow}
\end{figure*}
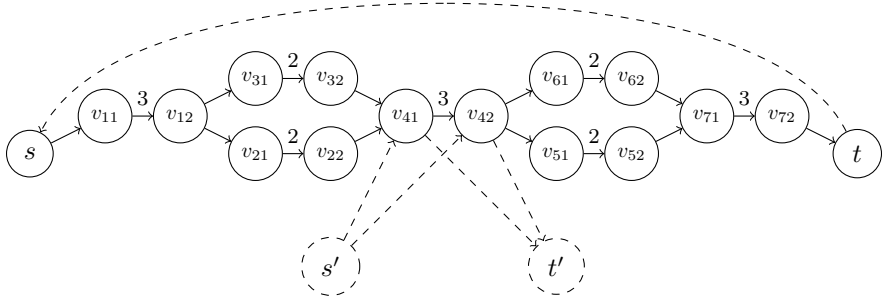

\begin{problem}[Minimum Circulation with Demand]
There is a graph $G(V,E)$, where $s,t$ are the source and terminal nodes. 
Every edge $uv$ in $G$ has a demand of $d(uv)$ and capacity of $c(uv)$. 
Compute minimal network flow from $s$ to $t$ that saturates all the edge 
demands.
\end{problem}

Readers are referred to \cite{kleinberg2006algorithm} for 
the details of the classical ``circulation with demand'' problem solution with max-flow. 

Alg.~\ref{alg:mindag} outlines the operations
used in this section. We note that the 
graph data structure needs to have the functionality of adding vertex and adding 
edge with edge demand, and $G'.flow[v_1][v_2]$ denotes the flow from $v_1$ to $v_2$ on $G'$.
\subsection{The Complete Allocation Algorithm}
The overall algorithm for robot allocation is described as in Alg.~\ref{alg:overall}.
To analyze the running time of Alg.~\ref{alg:overall} for a polygonal environment, 
we denote the environment complexity (number of vertices) as $n$. 
As mentioned in the previous section, the generalized boustrophedon decomposition takes $O(n\log n)$ time. 
As the generation of a node in the DAG comes from some chain insertion or deletion events, and an event would require one vertex to happen, 
there is at most $O(n)$ decomposed cells, i.e., nodes in the DAG. 
Similarly, adding edges between nodes comes from some chain insertion or deletion events, and the number of events is at most $O(n)$. 
So, both the number of edges and the number of nodes in the DAG is $O(n)$.
Moreover, it is easy to see that the DAG is a planar graph since a node corresponds to a decomposed cell, and there is an edge between nodes only if the two corresponding cells are adjacent.
The circulation with demand problem requires solving two max flow problem based on the DAG. 
If we use the push-relabel algorithm \cite{cheriyan1989analysis} that runs in $O(V^2\sqrt{E})$, solving the max flow for this problem will cost $O(n^{2.5})$, 
since $|V|, |E| = O(n)$.

\begin{algorithm}[ht]
\SetKwFunction{minsweep}{MinSweep}
\KwData{$P(t)$: a sweep schedule. $\mathcal W$: a compact workspace.
$\rho_0$: required sensing probability guarantee.
}
\KwResult{$n$: the minimum number of robot guards required. $plan$: the corresponding allocation plan}
\SetKwComment{comment}{\%}{}

\vspace{1mm}
Construct $\zeta$ based the sensing model and $\rho_0$\;
\vspace{1mm}

\begin{small}
\comment{Primitive function $\zeta$ is to compute the minimum number of robots required for a continuous line segment}
\end{small}
\vspace{1mm}
$dag\leftarrow \genbou(P, {\mathcal W})$;\;

\vspace{1mm}
$dag\gets \mindag(dag, \zeta)$;\;

\vspace{1mm}
$plan\gets dag$;\;

\vspace{1mm}
\begin{small}
\comment{The plan of the robots can be constructed based on the flow on the $dag$.} 
\end{small}

\vspace{1mm}
\Return{$dag.s.guards\_num$, plan};\;

\caption{\protect\minsweep{$P, {\mathcal{W}}, \rho_0$}: Computing Minimum Number of Robots for a Sweep Schedule}
\label{alg:overall}
\end{algorithm}

\begin{remark}
In the case of having a fixed number of robots, 
and the objective is to 
maximize the minimum coverage probability of a point in $\mathcal W$.
We can simply apply binary search on the minimum coverage probability we can
guarantee, where Alg.~\ref{alg:overall} 
can be used to decide whether some coverage probability probability can be guaranteed by the fixed number of robots.
\end{remark}



    


%% file: tex/expr.tex
\section{Simulation Evaluation}
In this section, we perform a numerical evaluation of our proposed method
with two goals: (1) confirms the scalability and (2) observe the behavior 
of the method as input parameters change. 
We implemented our proposed algorithms in C++. Dinic's algorithm is used to solve the max-flow problem \cite{dinitz1970algorithm} for simplicity. 
The methods are evaluated at an Intel\textsuperscript{\textregistered} Core\textsuperscript{TM} i5-10600K CPU at 4.1HZ.
For the three use cases (Fig.~\ref{fig:sweep} (b)(c)(d)), we
programmatically create a large number of test cases, with up to 
$6,000$ randomly generated polygonal obstacles. 
The number of vertices of each polygon ranges from 3 to 50, making the total vertices up to 
$100,000+$. 
Fig.~\ref{fig:cases} shows the largest problem instances for the experiments with 
around $100,000$ vertices.
\begin{figure}[ht]
    \centering
    \includegraphics[width=.95\linewidth]{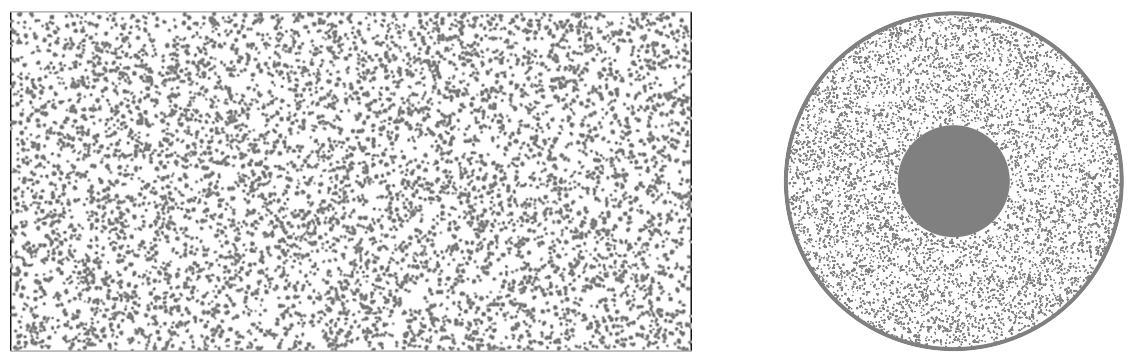}
    \caption{Examples of programmatically generated test environments with total 
    vertices around $100,000$.}
    \label{fig:cases}
\end{figure}

\textbf{Algorithm performance.} In a first set of evaluations, under an exponentially 
decaying sensing 
model, $\rho(r) = e^{-c\cdot r}$, we test the performance of our algorithm over 
the set of instances. Example allocation of robots along the sweep frontiers for 
the three cases are shown in Fig.~\ref{fig:sweep}(a) and Fig.~\ref{fig:simulations}. 
We limited the number of robots to be small so that the trajectories are more easily 
observed. It can be seen that the trajectories can vary significantly along the 
sweep frontiers for each case. 
\begin{figure}[ht]
    \centering
    \includegraphics[width=.45\linewidth]{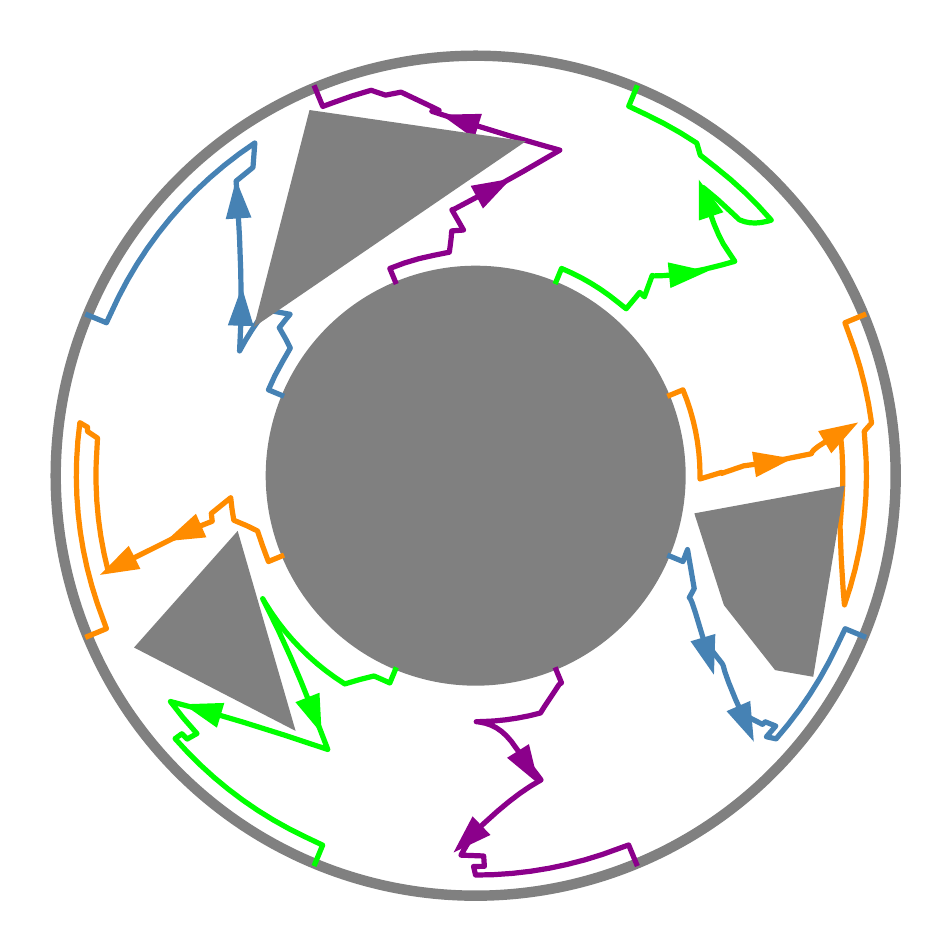}\hspace{2mm}
    \includegraphics[width=.45\linewidth]{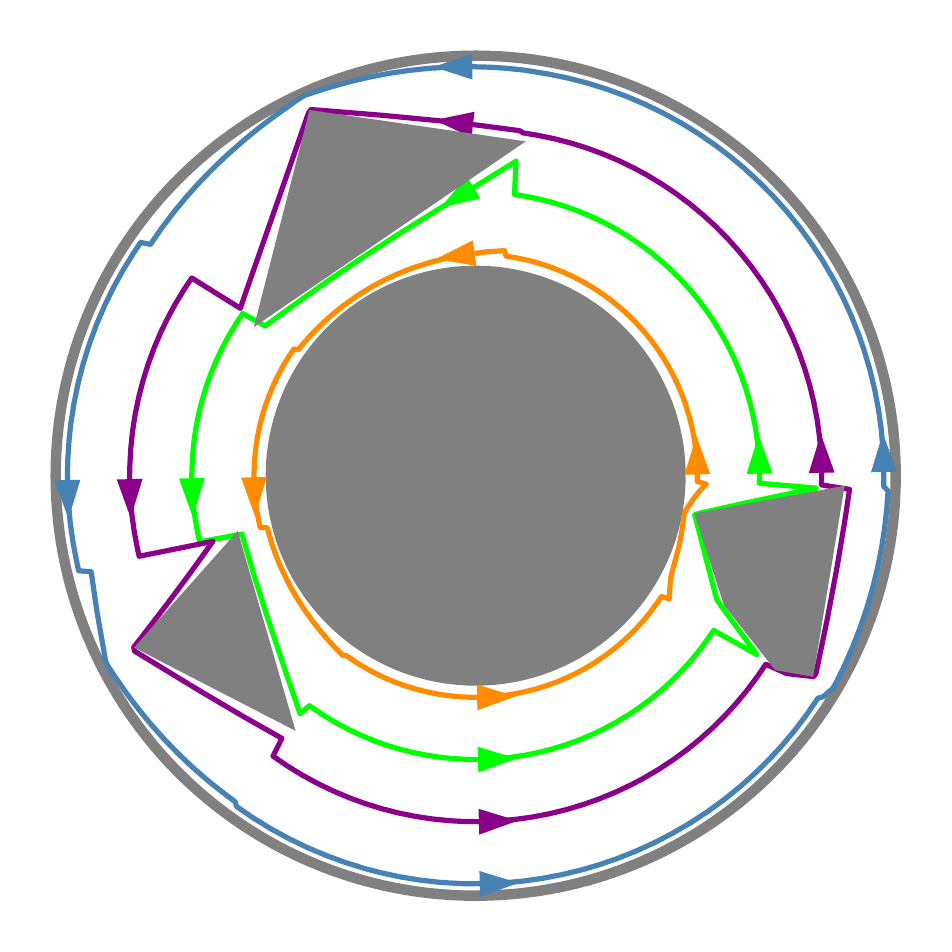}
    
    \caption{Example robot trajectories computed by our method for circular and radial 
    sweep use cases, respectively.
    }
    \label{fig:simulations}
\end{figure}

In Fig.~\ref{fig:simulations_runtime}, the computational performance is thoroughly 
evaluated. As can be observed, our method scales fairly well, taking less than two seconds
to handle all cases, even those involving over $100,000$ vertices. 
%
Interestingly, the experimental running time is almost linear, even using the 
less efficient Dinic's algorithm. 
We suspect the near linear running time is due to the fact that the DAG is 
a planar graph; studies show that max-flow for planar graphs can be computed 
in $O(n\log n)$ \cite{borradaile2009n}. 
Although the graph constructed when solving the ``circulation with demand'' 
problem is not planar, large portions of it are planar. 
This could reduce the actual time complexity of running the max-flow algorithm.

\begin{figure}[h]
    \centering
    \includegraphics[width=0.9\linewidth]{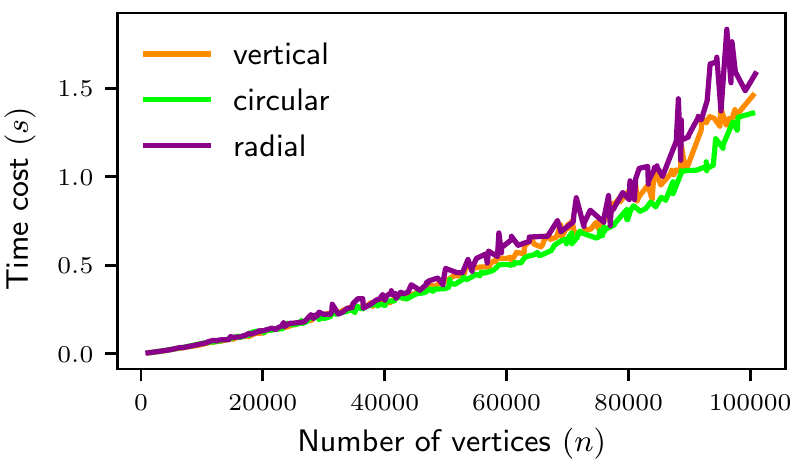}
    \caption{Running time in seconds with respect to environment complexity (number of vertices). }
    \label{fig:simulations_runtime}
\end{figure}

\textbf{Different environment settings.}
Lastly, we evaluate the impact of environmental changes, considering 
factors including the spatial distribution of polygonal obstacles as well as 
the size distribution of the obstacles. 
All in all, three settings are considered: (1) The polygons are regularly distributed 
and are of similar size, (2) the polygons are randomly distributed and are of similar size,
and (3) the polygons are regularly distributed, and their sizes can vary dramatically. 
The first and the third settings are illustrated in the top row of Fig.~\ref{fig:runtime_env}.
For these settings, we compare the time it takes to compute solutions for many polygonal
obstacles and also the number of robots required to achieve $80\%$ probabilistic guarantee 
under the exponential decay sensing model.
As we can observe, there is little difference as the settings change. 

\begin{figure}[h]
    \centering
    \hspace{5mm}
    \includegraphics[width = 0.44\linewidth]{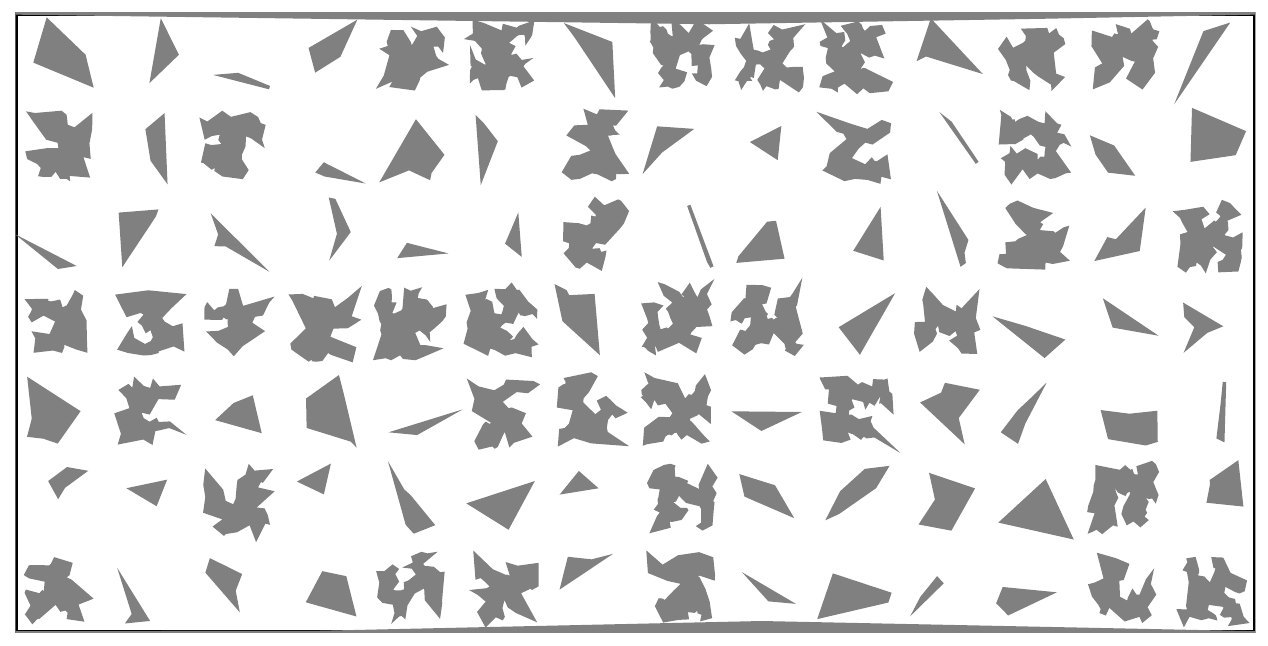}\hspace{1mm}
    \includegraphics[width = 0.44\linewidth]{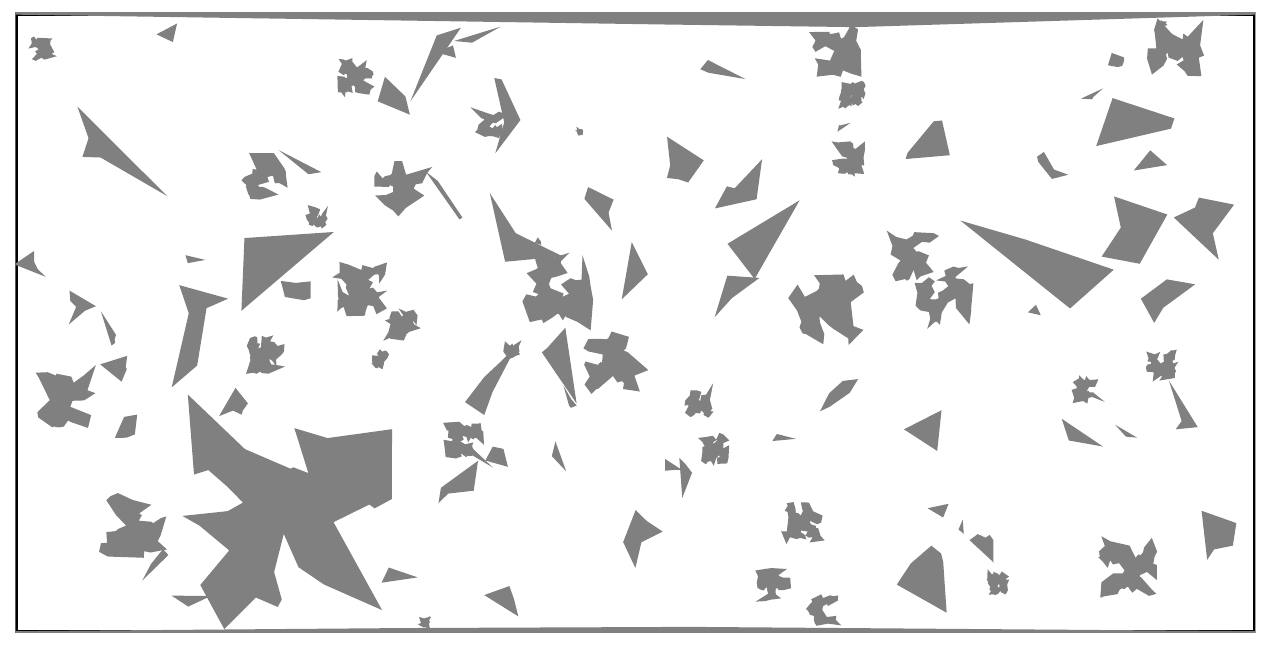}
    \vspace{3mm}
    
    \includegraphics[width=.47\linewidth]{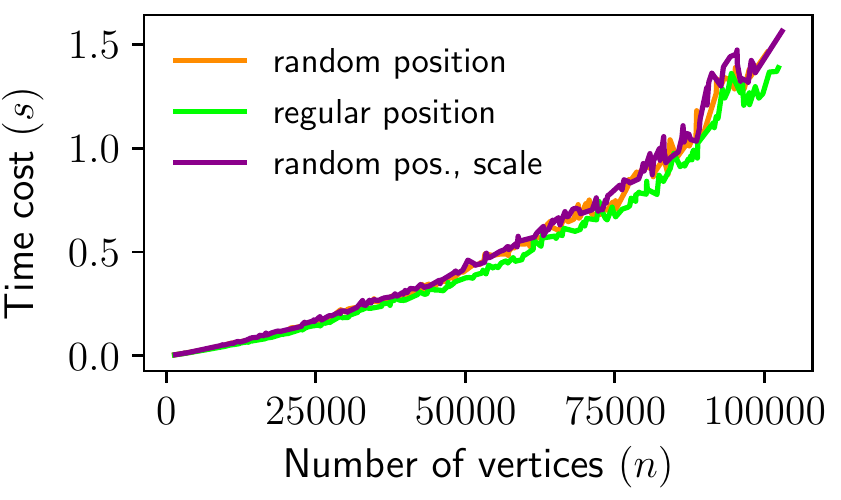}
    \includegraphics[width=.48\linewidth]{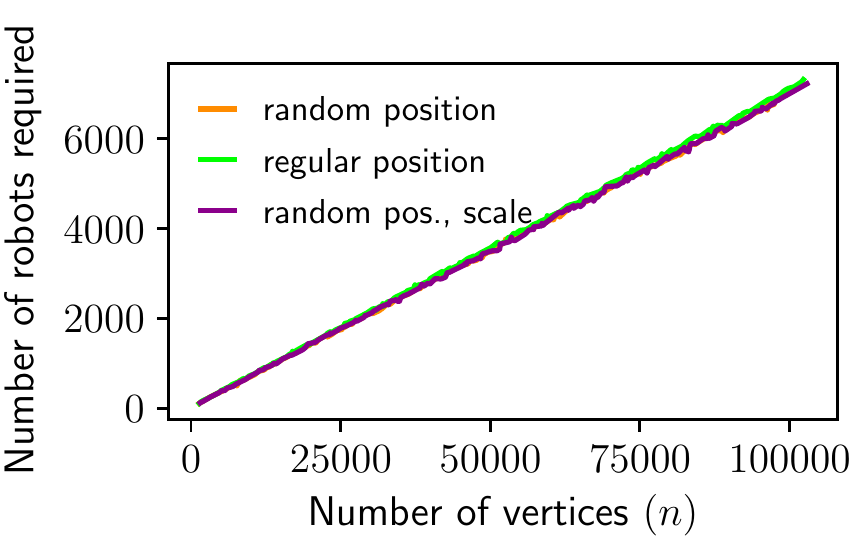}
    \caption{
    [top] Random instance with regularly distributed obstacles and instance with random obstacle scales and positions.
    [bottom] Running time for different randomly generated environments and the number of robots required for different randomly generated environments.}
    \label{fig:runtime_env}
\end{figure}


%% file: tex/conclusion.tex
\section{Conclusion }

In this paper, we studied the problem of allocating a minimum number of robots for a sweep schedule with a probabilistic line sensing model, where a desired level of scanning quality can be guaranteed. 
Towards this, a novel decomposition technique is proposed that generalizes the well-known boustrophedon decomposition. The decomposition leads naturally to  a transformation of the problem into a network-flow problem. Due to the decomposition and the transformation, our proposed algorithm runs in low polynomial time and even near $\Tilde{O}(n)$ in simulation experiments for polygonal environments, where $n$ is the complexity of the environment, measured as the number of vertices of polygons. Extensive simulation-based evaluation corroborates the effectiveness of our algorithm, which is applicable to multiple types of environments.  
